\documentclass[a4paper,11pt]{article}

\usepackage{amsmath,amssymb,amsthm}
\newtheorem{proposition}{Proposition}
\newtheorem{remark}{Remark}

\usepackage[a4paper, margin=25mm]{geometry}
\usepackage{setspace}
\usepackage{graphicx}
\usepackage{booktabs}
\usepackage{longtable}
\usepackage{array}
\usepackage{caption}
\usepackage{float}
\graphicspath{{fig/}}

\usepackage{algorithm}
\usepackage{algpseudocode}

\usepackage{xcolor}
\usepackage{hyperref}
\usepackage{url}
\usepackage{enumitem}

\usepackage[style=authoryear, backend=biber]{biblatex}
\begin{document}

\title{A Generalized-Bayes Perspective on Counterfactual Explanations: Posterior-Based Decision-Making and Evaluation}
\author{Keita Kinjo\\Faculty of Business, Kyoritsu Women's University}
\date{}
\maketitle

\begin{abstract}
Counterfactual explanations (CEs) enhance the interpretability of machine learning models by identifying the smallest change to an input required to obtain a desired output. Although CEs are conventionally formulated as a distance-minimization problem, the theoretical basis of this formulation has received limited attention. We show that a distance-minimization-based CE is mathematically equivalent to the maximum a posteriori (MAP) estimate of a Gibbs posterior within the generalized Bayes framework, specifically when a distance-based prior is used. We call this formulation the Distance-Prior Generalized Bayes CE (DP-GBCE). Building on this posterior perspective, we introduce two decision rules beyond MAP within a unified framework: a Bayes decision that minimizes expected decision loss and CVaR-CE, a risk-averse decision rule. We also propose an extension that uses Bayesian model weights to mix the posterior distributions of multiple models, thereby accounting for model multiplicity, where several models have comparable predictive performance. Finally, we define metrics for evaluating both individual CEs and the posterior distribution as a whole, and use experiments on simulated data and Google Trends data to quantify the trade-offs among the decision rules.
\end{abstract}

\section{Introduction}

Counterfactual explanations have attracted considerable attention as a means of addressing the interpretability challenge in machine learning \parencite{guidotti2024counterfactual, verma2024counterfactual, karimi2022survey}. Given a trained black-box model, a counterfactual explanation identifies the smallest change to a particular instance's attributes required to obtain a desired prediction. It thereby indicates which attributes affect the prediction and what actions could change the outcome. For example, if a machine learning model denies an individual insurance coverage, a counterfactual explanation can identify the changes in attributes such as income or occupation that would result in approval. This approach is also known as algorithmic recourse \parencite{ustun2019actionable, karimi2020algorithmic}, and its connection to adversarial examples has likewise been noted \parencite{karimi2022survey, verma2024counterfactual, pawelczyk2022exploring}.

CE methods typically minimize the cost of changing the original instance into the counterfactual, thereby encouraging similarity between the two \parencite{wachter2017counterfactual, guidotti2024counterfactual, verma2024counterfactual, karimi2022survey}.

One rationale for this formulation draws on Lewis's notion of the closest possible world, constructed by making only the minimal changes to the actual world needed to evaluate a causal effect \parencite{lewis2013counterfactuals}. Despite this philosophical foundation, theoretical discussion has been limited. Although a large required change may intuitively suggest that the corresponding variable is important, the theoretical basis for this interpretation has rarely been examined.

Generalized Bayes extends conventional Bayesian inference by allowing a posterior distribution to be constructed from a loss function even when a likelihood is unavailable, as is often the case in machine learning \parencite{bissiri2016general}. It constructs a Gibbs posterior through exponential weighting based on the loss rather than on a likelihood.

In this study, we show that the CE obtained through distance minimization corresponds to the MAP estimate of a Gibbs posterior over CEs, thereby providing a probabilistic justification for the existing cost-minimization formulation. We further propose a posterior distribution over CEs that incorporates model uncertainty, and we present decision rules other than MAP for the resulting CE posterior. In addition, we propose metrics for evaluating individual CEs and the posterior distribution as a whole, and compare their performance. By clarifying the connection between CEs and Bayesian inference -- generalized Bayes in particular -- this study opens up alternative approaches to challenges such as robustness, as well as applications that draw on Bayesian insights.

Several previous studies have also examined CEs within a Bayesian framework (see Section 4).

This study differs from the existing literature in four main respects. First, it establishes the equivalence between CE distance minimization and MAP estimation of a Gibbs posterior, thereby providing existing methods with a probabilistic foundation. Second, it offers a model-agnostic framework that requires neither an explicit generative model nor a specified likelihood and can be applied to any learner -- differentiable or otherwise -- as long as its loss can be evaluated. This contrasts with \textcite{raman2023bayesian}, which requires a generative probability model and gradient information. Third, it places a probability distribution (the Gibbs posterior) over counterfactual candidates and enables MAP, Bayes, and CVaR-CE decision rules to be compared within a unified framework, allowing robust decision-making under uncertainty in the success region. This contrasts with \textcite{nguyen2022robust}, which seeks a single counterfactual point. Fourth, as a natural extension for addressing model multiplicity -- the coexistence of multiple models with comparable predictive performance -- it yields a distributional CE that incorporates model uncertainty by mixing the posterior distributions of the individual models using Bayesian model weights (see Section 2.4).

The remainder of this paper is organized as follows. Section 2 describes the proposed method, Section 3 presents the experiments, Section 4 details related work, and Section 5 provides a discussion.

\section{Counterfactual Explanations via Generalized Bayes}

\subsection{Counterfactual Explanation Formulation}

Let $D = \{(x_i,y_i)\}_{i=1}^{n}$, where $x_i \in \mathcal{X} \subseteq \mathbb{R}^{m}$ is an input and $y_i \in \mathcal{Y}$ is its target value. We take $\mathcal{Y}=\mathbb{R}$ for regression and $\mathcal{Y}=\{0,1\}$ for binary classification. Let $f:\mathcal{X} \rightarrow \mathcal{Y}$ denote a prediction model trained on $D$. We write $x_b$ for the input to be explained, $y=f(x_b)$ for its predicted outcome, and $y^* \in \mathcal{Y}$ for the desired output. A counterfactual explanation for $x_b$ is denoted by $\tilde{x} \in \mathcal{X}$. The function $d(x_b,\tilde{x})$ represents the cost or distance associated with the change, and $\ell(f(\tilde{x}),y^*)$ measures failure to attain the desired output. Squared-error loss is commonly used for continuous outcomes and cross-entropy loss for binary outcomes.

The CE problem is then formulated as follows \parencite{wachter2017counterfactual}:
\begin{equation}
    \tilde{x}^{CE} \in \underset{\tilde{x} \in \mathcal{X}}{\operatorname{argmin}}\left[ \lambda\ell(f(\tilde{x}),y^{*}) + d(x_{b},\ \tilde{x}) \right]
\end{equation}
The problem is therefore to find a $\tilde{x}$ that both brings the prediction close to the desired output $y^*$ (small loss) and remains close to the original input $x_b$ (small distance). Here, $\lambda \geq 0$ controls the weight assigned to the loss. A fixed $\lambda$ is widely used for three reasons. First, it is often difficult to solve either the constrained problem of minimizing $d(x_b,\tilde{x})$ subject to $\ell(f(\tilde{x}),y^*) \leq \varepsilon$ or the bi-objective problem of jointly minimizing $\ell$ and $d$; many implementations therefore use a fixed-weight penalized objective. Second, an exact constraint may yield no feasible solution for nonlinear models or noisy predictions, whereas a fixed $\lambda$ provides a soft constraint. Third, $\lambda$ has a natural interpretation as the trade-off between goal attainment and feasibility, and thus expresses the degree of tolerance permitted in the explanation \parencite{guidotti2024counterfactual, verma2024counterfactual, karimi2022survey}.

\subsection{Generalized Bayes and the Gibbs Posterior}

We next introduce the generalized Bayes framework and use it to recast the CE problem from a Bayesian perspective.

Generalized Bayes updates a prior distribution through exponential weighting based on a task-specific loss function, without requiring the likelihood assumed in conventional Bayesian inference \parencite{bissiri2016general}.

To review conventional Bayesian inference briefly, let $D$ denote the observed data and $\theta \in \Theta$ the unknown quantity of interest. Given a prior $p(\theta)$ and a probability model $p(y|x,\theta)$, the posterior is
\begin{equation}
    p(\theta|D) \propto p(D|\theta)p(\theta)
\end{equation}
Here, $p(D|\theta)$ is the likelihood and, under conditional independence, can be written as $\prod_{i = 1}^{n}p(y_i|x_i,\theta)$.

In generalized Bayes, a loss function $L(D,\theta) \in \mathbb{R}^{+}$ that evaluates $\theta$ against $D$ replaces the likelihood, yielding the following generalized Bayes posterior (or Gibbs posterior):
\begin{equation}
    p(\theta|D) \propto \exp( - \eta L(D,\theta))p(\theta)
\end{equation}
Here, $\eta$ is called the learning rate and controls the concentration of the posterior distribution.

Setting $L(D,\theta) = -\log p(D|\theta)$ and $\eta = 1$ recovers conventional Bayesian inference as a special case. Because the MAP estimate minimizes an objective defined by the loss and the prior, it can also be interpreted as a regularized M-estimator. Within PAC-Bayes theory, the Gibbs posterior arises as the distribution that optimally trades off expected loss against Kullback--Leibler divergence from the prior, and it can carry generalization guarantees even without a likelihood model \parencite{guedj2019primer}. These guarantees concern estimation of a parameter $\theta$; we do not claim that an analogous guarantee applies to the CE posterior over $\tilde{x}$ introduced in Section 2.3.

\subsection{Counterfactual Explanation as a Gibbs Posterior}

We now apply this framework to CEs and derive a Gibbs posterior over counterfactuals.

Let $\tilde{x}$ be the object of inference and let $p(\tilde{x}|x_b)$ be its prior distribution conditional on $x_b$. Using $\ell(f(\tilde{x}),y^*)$ as the loss yields
\begin{equation}
    p(\tilde{x}|x_{b},y^{*}) \propto \exp( - \eta\ell(f(\tilde{x}),y^{*}))p(\tilde{x}|x_{b})
\end{equation}
We refer to this formulation as the Generalized Bayes CE (GBCE).

Choosing the distance-based prior $p(\tilde{x}|x_{b}) \propto \exp(-d(x_b, \tilde{x}))$ gives
\begin{equation}
    p(\tilde{x}|x_{b},y^{*}) \propto \exp\!\left( - \eta\ell(f(\tilde{x}),y^{*}) - d(x_{b},\tilde{x})\right)
\end{equation}
We refer to this formulation as the Distance-Prior Generalized Bayes CE (DP-GBCE). Hereafter, unless otherwise noted, ``Gibbs posterior'' or ``Gibbs'' refers to the DP-GBCE. Algorithm~\ref{alg:dpgbce} in Section 2.6 summarizes the overall computational procedure: it approximates the DP-GBCE (and the ModelUnc construction described in Section 2.4) via importance sampling, applies the decision rules, and computes the evaluation metrics. The distribution $p(\tilde{x}|x_b)$ is a user-specified prior over counterfactual candidates and does not represent a data-generating process. Depending on the distance function, the Gibbs-type (or Boltzmann-type) prior $\exp(-d(x_b,\tilde{x}))$ includes the normal distribution (for squared Euclidean distance) and the Laplace distribution (for $L_1$ distance) as special cases.

\begin{proposition}
Under the above setting, the MAP estimator of the generalized-Bayes posterior reduces to the following CE optimization problem:
\begin{equation}
\label{eq:map}
    \hat{x}^{MAP} = \underset{\tilde{x}}{\operatorname{argmax}}\ p(\tilde{x}|x_{b},y^{*})
    = \underset{\tilde{x}}{\operatorname{argmin}}\left[ \eta\,\ell(f(\tilde{x}),y^{*}) + d(x_{b},\tilde{x}) \right]
\end{equation}
\end{proposition}

\begin{proof}
$p(\tilde{x}|x_{b},y^{*}) \propto \exp(-\eta\ell(f(\tilde{x}),y^*)-d(x_b,\tilde{x}))$, and $\operatorname{argmax}$ is invariant under multiplication by a positive constant, so
$\hat{x}^{MAP} = \operatorname{argmax}_{\tilde{x}}\ p(\tilde{x}|x_{b},y^{*})
  = \operatorname{argmax}_{\tilde{x}}\ \exp(-\eta\ell(f(\tilde{x}),y^*)-d(x_b,\tilde{x}))$.
Since the exponential function is monotonically increasing, this is equivalent to $\operatorname{argmin}_{\tilde{x}}\left[\eta\ell(f(\tilde{x}),y^*)+d(x_b,\tilde{x})\right]$.
\end{proof}

This proposition shows that distance minimization in conventional CE is equivalent to MAP estimation of a Gibbs posterior within the generalized-Bayes framework. This equivalence is not merely a formal correspondence; it strengthens the theoretical justification of CE in the following sense. By reformulating CE as a Gibbs posterior, one can naturally introduce, beyond the single-point MAP estimate, a variety of decision rules that use the full posterior distribution (e.g., Bayes decision, CVaR-CE) as well as distribution-level evaluation metrics (e.g., success probability, stability). In other words, the justification put forward in this study is the structural claim that distance minimization is a special case of probabilistic inference, which in turn provides the theoretical grounds for applying the rich toolkit of Bayesian methods to CE.

This expression admits the following interpretation.

\begin{remark}[Relation between $\lambda$ and $\eta$]
There are two possible views on the relationship between $\lambda$ and $\eta$. (1) Under the view that identifies $\eta = \lambda$, $\tilde{x}^{CE} = \hat{x}^{MAP}$ holds in a set-valued sense. (2) Under the view that treats $\lambda$ and $\eta$ as parameters with distinct roles and fixes $\lambda = \eta = 1$, the role of the trade-off coefficient $\lambda$ is absorbed into the scale of the prior (the variance $\sigma^2$ of the Gaussian prior); setting $d(x_b,\tilde{x}) = \frac{1}{2\sigma^2}\|\tilde{x}-x_b\|^2$ likewise makes $\tilde{x}^{CE} = \hat{x}^{MAP}$ hold.
\end{remark}

For consistency of exposition and comparability with previous studies, we use $\eta=1$ as the default. This corresponds to view (2), in which the trade-off is represented through the relative scales of the loss and the prior. The behavior for $\eta \neq 1$ is examined in the $\eta$-sensitivity analysis in Section 3 (Figure 3). Under view (1), varying $\eta$ changes the spread of the posterior distribution. This interpretation also suggests extensions such as methods for selecting $\eta$ within generalized Bayes.

For the Gibbs posterior to be well defined as a probability distribution, however, the unnormalized density $\exp(-\eta\ell - d)$ must be integrable. If the loss function is nonnegative and measurable, then $\exp(-\eta\ell(f(\tilde{x}),y^{*})) \leq 1$, so the condition $\int_{\mathcal{X}}\exp(-d(x_{b},\tilde{x}))\,d\tilde{x} < \infty$ on the prior side (which holds for the norm-based distances used in this study) is a sufficient condition. Note that the loss function need not correspond to a likelihood.

The epistemological status of the posterior $p(\tilde{x}|x_b,y^*)$ introduced in this section is discussed in Section 5.

\subsection{Counterfactual Explanations Incorporating Model Uncertainty}

As a natural extension of the CE presented in Section 2.3, we now describe a formulation that accounts for model uncertainty. Here, model uncertainty broadly encompasses uncertainty about the configuration $\theta$ of the prediction model, including both the choice of model class (e.g., LightGBM, XGBoost, or Random Forest) and the parameter values within a model class. Our formulation accommodates both discrete $\theta$ (model-class choice) and continuous $\theta$ (model parameters). When multiple distinct models have comparable predictive performance, a phenomenon known as model multiplicity \parencite{jiang2024robust, pawelczyk2020counterfactual, kinjo2025robust}, a CE derived from one model may be invalid under another.

Suppose that we place a prior on $\theta$ and use a loss $L(D,\theta)$ in place of a likelihood to construct the generalized Bayes posterior $p(\theta|D) \propto \exp(-\eta' L(D,\theta))p(\theta)$. Integrating over this posterior gives
\begin{equation}
\label{eq:genbayes-ce}
p(\tilde{x}|x_{b},y^{*},D) \propto \int \exp(-\eta\,\ell(f_\theta(\tilde{x}),y^{*}))\,p(\tilde{x}|x_b)\,p(\theta|D)\,d\theta
\end{equation}
This is a fully generalized Bayesian CE that requires no likelihood; we refer to it as the Fully Generalized Bayes CE (FG-GBCE). Whereas the GBCE of Section 2.3 generalizes the prior $p(\tilde{x}|x_b)$, the FG-GBCE extends the framework along a separate dimension by incorporating uncertainty in the prediction model $\theta$. This construction, however, has a nested structure: it requires sampling from the posterior over $\theta$ and then over $\tilde{x}$ for each $\theta$, making it computationally expensive. Moreover, tree-based ensemble models such as LightGBM and XGBoost do not have a fixed-dimensional continuous parameter vector, which makes a direct application of this formulation with continuous $\theta$ difficult.

We therefore adopt, in this study, a practical special case in which $\theta$ is treated as a discrete variable over a finite set of model classes $\{1,\ldots,K\}$. Taking a finite set of models $\left\{ f^{(1)},\ldots,f^{(K)} \right\}$, selected by some procedure, as candidates, and setting $p(\theta = k|D) = w_k$, the integral in Eq.~\eqref{eq:genbayes-ce} reduces to a sum, yielding
\begin{equation}
    \label{eq:model-mixture}
    p(\tilde{x}|x_{b},y^{*},D) \propto \sum_{k = 1}^{K}{w_{k}\exp( - \eta\,\ell(f^{(k)}(\tilde{x}),y^{*}))\,p(\tilde{x}|x_{b})}
\end{equation}
We refer to this formulation as ModelUnc (Model-Uncertainty DP-GBCE). It is the special case of FG-GBCE in which $\theta$ is discrete over a finite set of model classes. When accounting for model multiplicity, we set $w_k=1/K$; when the models differ in performance, we set $w_k \propto \exp(-\gamma R_k)$, where $R_k$ is the cross-validation loss.

Note that this differs from approaches such as stacking, in which multiple models are combined into a single model beforehand.

Furthermore, if the losses $L(D,\theta)$ and $\ell(f_\theta(\tilde{x}),y^*)$ in the FG-GBCE (Eq.~\eqref{eq:genbayes-ce}) are each made to correspond to a likelihood, then, when $\theta$ represents the coefficients of a continuous parametric model, one obtains a fully Bayesian counterfactual explanation in the conventional sense (a Fully Bayesian CE) that jointly randomizes the prediction model's parameters and the counterfactual. This special case, however, requires strong generative-model assumptions (the design of a likelihood). This point is also addressed in the discussion in Section 5.

The proposed framework can therefore accommodate various forms of uncertainty and constraints, including model uncertainty.

\subsection{Applications of the Posterior Distribution}

The resulting Gibbs posterior $p(\tilde{x}|x_b,y^*)$ provides more than a single optimal solution: the full distribution can support a variety of decision-making and evaluation tasks. We consider three such uses: (1) selecting a CE from the posterior, (2) evaluating the resulting point estimate, and (3) evaluating the posterior distribution itself.

\subsubsection{Selecting a CE from the Posterior Distribution}

Several methods can be used to select a point-valued CE from the posterior. The appropriate decision rule depends on the problem structure, including the shape of the success region and the cost of failure. We consider three main rules, representing distinct decision-making perspectives: correspondence with conventional CE (MAP), point summarization through minimization of posterior expected decision loss (Bayes decision), and risk aversion (CVaR-CE).

\begin{description}
  \item[(a) MAP decision (the mode)]
    $\tilde{x}^{MAP} = \operatorname{argmax}_{\tilde{x}}\ p(\tilde{x}|x_b,y^*)$. Appropriate when the success region is unimodal, or, even if multimodal, when the selected mode has a clear meaning.
  \item[(b) Bayes decision (minimizing the expected decision loss)]
    \begin{equation}
      \tilde{x}^{Bayes} \in \underset{z \in \mathcal{X}}{\operatorname{argmin}}\ \mathbb{E}_{\tilde{x}\sim p( \tilde{x} | x_{b},y^{*} )}[\Delta(z,\tilde{x})]
    \end{equation}
    Here, $\Delta(z,\tilde{x})$ is a decision loss on the input space that measures the discrepancy between the decision $z$ and a counterfactual $\tilde{x}$ drawn from the posterior; note that it is not the goal-attainment loss $\ell$ itself. When $\Delta = \|\cdot\|^{2}$, the solution is the posterior mean; when $\Delta = \|\cdot\|$, it is the geometric median. Note, however, that when the posterior distribution is multimodal or curved, the posterior mean can fall in a low-density region that does not attain the goal -- a known limitation of point summarization. In such cases, an alternative is the medoid, which constrains the decision to lie within the posterior samples.
  \item[(c) Risk-averse decision (CVaR-CE)]
    Let $\mathcal{D}_{r}$ denote the perturbation distribution representing the error incurred when executing the CE (in this study, the isotropic normal distribution $\mathcal{N}(0,\sigma_{\delta}^{2}I)$; see Section 3 for the specific settings), and let $L_{\delta}(z) = \ell(f(z+\delta),y^{*})$ (with $\delta \sim \mathcal{D}_{r}$) denote the loss under perturbation. We define CVaR-CE as the point that minimizes the conditional expectation (CVaR) of the upper $(1-\tau)$ tail of this perturbed loss:
    \begin{equation}
      \tilde{x}^{CVaR} \in \underset{z \in \mathcal{X}}{\operatorname{argmin}}\ \mathrm{CVaR}_{\tau}\!\left( L_{\delta}(z) \right),
      \qquad
      \mathrm{CVaR}_{\tau}\!\left( L_{\delta}(z) \right) = \mathbb{E}_{\delta\sim\mathcal{D}_{r}}\!\left[ L_{\delta}(z) \,\middle|\, L_{\delta}(z) \geq \mathrm{VaR}_{\tau}\!\left(L_{\delta}(z)\right) \right]
    \end{equation}
    Here, $\mathrm{VaR}_{\tau}$ is the $\tau$-quantile of $L_{\delta}(z)$. In implementation, we restrict the candidate points $z$ to samples from the posterior $p(\tilde{x}|x_b,y^*)$ (see Section 3). This risk-averse rule limits exposure to poor outcomes under execution noise when success is uncertain or failure is costly.
\end{description}

Beyond these, many other approaches are conceivable, such as using the medoid; extracting a set of CEs using the HPD (highest posterior density) region when diversity is to be ensured; chance-constrained decisions; and sample-based decisions that sample directly from the distribution.

\subsubsection{Pointwise Evaluation Metrics for Individual CEs}

We define metrics for evaluating the quality of a single-point CE $\tilde{x}$ obtained by a decision rule. These are standard metrics based on existing studies \parencite{guidotti2024counterfactual, verma2024counterfactual, karimi2022survey}, and we use four: achieved loss, change distance, robustness, and plausibility. These four metrics respectively correspond to the main requirements demanded of a CE: goal attainment, minimality of change, stability at execution time, and feasibility with respect to the data.

\begin{align}
  L_{pt}(\tilde{x}) &= \ell(f(\tilde{x}),y^{*}) \\
  D_{pt}(\tilde{x}) &= d(\tilde{x},x_{b}) \\
  Rb(\tilde{x}) &= P_{\delta\sim\mathcal{D}_{r}}[\ell(f(\tilde{x}+\delta),y^{*}) \leq \varepsilon_{rb}] \\
  Plu(\tilde{x}) &= \frac{1}{q}\sum_{i \in N_{q}(\tilde{x})}\left\| \tilde{x} - x_{i} \right\|
\end{align}

$L_{pt}$ represents the degree of attainment of the desired output $y^*$; smaller is better. $D_{pt}$ represents the amount of change from the original input; smaller is better. $Rb$ is the probability that the prediction remains within the success region when a small perturbation $\delta \sim \mathcal{D}_{r}$ (the same perturbation distribution as in Section 2.5.1(c)) is added to $\tilde{x}$, and it reflects the stability of the CE (higher is better). $Plu$ is the average distance to the $q$ nearest neighbors in the training data, measuring whether the CE is a plausible point with respect to the data distribution (plausibility; smaller is better).

\subsubsection{Evaluation Metrics for the Posterior Distribution Itself}

Independently of the decision rule, it is also important to evaluate the posterior distribution $p(\tilde{x}|x_b,y^*)$ itself. The following metrics quantify the posterior probability of success, tail risk, dispersion, and variable-specific changes.

\begin{align}
  SP &= P_{\tilde{x}\sim p}[\ell(f(\tilde{x}),y^{*}) \leq \varepsilon_{sp}] \quad \text{(success probability: higher is better)}\\
  Tail &= \inf\{t:P(\ell(f(\tilde{x}),y^{*}) \leq t) \geq 1-\alpha\} \quad \text{(tail of the distribution: smaller is better)}\\
  Stability &= \mathrm{tr}[\mathrm{Cov}(\tilde{x})] \quad \text{(stability: smaller is better)}\\
  VarImp_j &= \mathbb{E}[|\tilde{x}_{j} - x_{b,j}|] \quad \text{(distributional variable importance)}
\end{align}

$SP$ is the proportion of posterior samples satisfying the success condition $\ell \leq \varepsilon_{sp}$, and it indicates how well the posterior distribution covers the success region. $Tail$ is the $(1-\alpha)$-quantile of the loss, i.e., $\mathrm{VaR}_{1-\alpha}$, and it captures the heaviness of the distribution's tail (the worst-case loss level). $Stability$ is the trace of the covariance matrix of the posterior distribution, representing the magnitude of the CE's variance -- that is, the uniqueness and stability of the solution. $VarImp_j$ is the expected amount of change in each variable $j$, indicating at the distribution level which variables are important in the counterfactual.

\subsection{Overview of the Computational Procedure}

Algorithm~\ref{alg:dpgbce} summarizes the full procedure described above, from the construction of the DP-GBCE (and the ModelUnc of Section 2.4) through the application of the decision rules to the computation of the evaluation metrics. The experiments in Section 3 follow this procedure, performing an approximate computation via importance sampling.

\begin{algorithm}[H]
\caption{Computational procedure for DP-GBCE (and ModelUnc)}
\label{alg:dpgbce}
\begin{algorithmic}[1]
\Require Base point $x_b$, target value $y^*$, prediction model $f$ (for ModelUnc, a set of models $\{f^{(k)}\}_{k=1}^{K}$ with weights $\{w_k\}$), loss $\ell$, temperature $\eta$, proposal-distribution scale $\sigma$ (corresponding to the distance prior $d(x_b,\tilde{x})=\frac{1}{2\sigma^2}\|\tilde{x}-x_b\|^2$), number of candidates $N_{\mathrm{cand}}$, number of samples $N_{\mathrm{samp}}$
\State Generate candidate points $\{\tilde{x}_i\}_{i=1}^{N_{\mathrm{cand}}} \sim \mathcal{N}(x_b,\sigma^2 I)$ from the proposal distribution \Comment{Match the proposal distribution to the distance prior $\propto\exp(-d(x_b,\tilde{x}))$}
\For{$i = 1,\ldots,N_{\mathrm{cand}}$}
  \If{single model (DP-GBCE)}
    \State $u_i \leftarrow \exp\!\left(-\eta\,\ell(f(\tilde{x}_i),y^{*})\right)$
  \Else{ (multiple models, ModelUnc)}
    \State $u_i \leftarrow \textstyle\sum_{k=1}^{K} w_k \exp\!\left(-\eta\,\ell(f^{(k)}(\tilde{x}_i),y^{*})\right)$
  \EndIf
\EndFor
\State Normalize the weights $\{u_i\}$ and draw $N_{\mathrm{samp}}$ samples with replacement via importance resampling to obtain posterior samples $\{\tilde{x}_s\}_{s=1}^{N_{\mathrm{samp}}}$
\State Apply a decision rule (MAP, Bayes decision, or CVaR-CE) to determine a single-point CE (Section 2.5.1)
\State Compute pointwise evaluation metrics for the selected CE (Section 2.5.2)
\State Compute distributional evaluation metrics for $\{\tilde{x}_s\}$ (Section 2.5.3)
\State \Return the CE selected by each decision rule and the corresponding evaluation metrics
\end{algorithmic}
\end{algorithm}

\section{Empirical Evaluation}

This section evaluates the proposed method on simulated and real data. Using two-dimensional simulated data, we first examine whether the DP-GBCE and model-uncertainty posteriors concentrate in the success region (Figures~\ref{fig:fig1} and \ref{fig:fig2}) and compare the trade-offs among the MAP, Bayes, and CVaR-CE decision rules (Table~\ref{tab:table1}). We then examine whether similar trade-offs arise in ten dimensions (Table~\ref{tab:table2}) and analyze the sensitivity of the decision metrics to the temperature parameter $\eta$ (Figure~\ref{fig:fig3}). Finally, we evaluate whether these findings persist under practical conditions using real data. Together, these experiments illustrate the additional information provided by a distributional approach relative to point-estimate CE methods.

\subsection{Simulated Data}

We consider a regression task. The two-dimensional data ($d=2$, $n=3{,}000$) were generated from the following nonlinear function:
\[
  y = 2.0\sin(x_1) + 0.8\,x_2^2 - 1.2\,x_1 x_2 + \varepsilon, \quad \varepsilon \sim \mathcal{N}(0,\, 0.3^2)
\]
The ten-dimensional data ($d=10$, $n=4{,}000$) were generated from
\[
  y = 1.5\sin(x_1) + 0.8\,x_2^2 - x_1 x_3 + 0.5\,x_4 + 0.3\sum_{j=5}^{7}x_j + \varepsilon, \quad \varepsilon \sim \mathcal{N}(0,\, 0.5^2)
\]
In both cases, the input variables were generated independently from $\mathcal{N}(0,1)$. The target value $y^*$ was set to the 90th percentile of the training outcomes, and squared-error loss, $\ell(\hat{y}, y^*) = (\hat{y} - y^*)^2$, was used.

To train the prediction models used to compute the CEs, we applied FLAML AutoML with eight candidate estimators: LightGBM, XGBoost, depth-limited XGBoost, HistGradientBoosting, Random Forest, ExtraTrees, ElasticNet, and SGD. The evaluation metric was MSE, the search-time budget for each estimator was capped at 25 seconds, and the random seed was fixed at 42 throughout. Model selection followed FLAML's default setting (\texttt{eval\_method="auto"}), which uses five-fold cross-validation at the data scale considered here. The best-performing model was used for the Gibbs posterior, whereas the top $K=3$ models by MSE were used for the model-uncertainty CE. The latter used equal weights, $w_k=1/K$, in Eq.~\eqref{eq:model-mixture}; we did not examine performance-based weights $w_k \propto \exp(-\gamma R_k)$ in these experiments. We set $\eta=1$.

Table~\ref{tab:leaderboard_sim} shows the top three models selected by AutoML.

\begin{table}[H]
  \centering
  \caption{Simulated data: top three models selected by AutoML (ordered by MSE)}
  \label{tab:leaderboard_sim}
  \begin{tabular}{llr}
    \toprule
    Data & Model & MSE \\
    \midrule
    2D & ExtraTrees & 0.146 \\
     & XGBoost (depth-limited) & 0.152 \\
     & LightGBM & 0.157 \\
    \midrule
    10D & LightGBM & 0.426 \\
     & XGBoost (depth-limited) & 0.436 \\
     & XGBoost & 0.437 \\
    \bottomrule
  \end{tabular}
\end{table}

We approximated the Gibbs posterior by importance sampling. Candidate points were generated from the distance-based Gaussian prior centered at $x_b$, weighted by $\exp(-\eta\ell)$, and resampled. We used $N_{\mathrm{cand}}=2\times10^4$ and $N_{\mathrm{samp}}=2\times10^3$ (proposal scale $\sigma=1.0$) in 2D, and $N_{\mathrm{cand}}=3.5\times10^4$ and $N_{\mathrm{samp}}=3\times10^3$ ($\sigma=1.0$) in 10D. The Bayes decision (Section 2.5.1(b)) was computed with decision loss $\Delta=\|\cdot\|^{2}$, taking the sample mean (posterior mean) of the posterior samples as the solution; we refer to this as the Mean decision in the tables and figures below. For CVaR-CE, we set $\tau=0.9$, thereby minimizing the average loss in the worst 10\% of outcomes. We used $n_{\mathrm{perturb}}=64$ perturbations drawn from $\delta \sim \mathcal{N}(0,\sigma_\delta^2 I)$ (2D: $\sigma_\delta=0.2$; 10D: $\sigma_\delta=0.15$) and evaluated at most 800 candidate points in 2D and 1,200 in 10D. We approximated the MAP decision by selecting the posterior sample with the highest estimated density. Density was estimated by Gaussian kernel density estimation (KDE) in 2D and by a $k$-nearest-neighbor score in 10D, where KDE is less reliable because of the curse of dimensionality. Specifically, the score was $r_k^{-d}$, where $r_k$ is the distance to the $k$th neighbor and $k=40$. The resulting MAP should therefore be regarded as an approximation, particularly in 10D.

To assess the proposition empirically, we also included a direct-optimization baseline (DirectOpt) that minimizes $\lambda\ell(f(\tilde{x}),y^*)+d(x_b,\tilde{x})$, with $d(x_b,\tilde{x})=\|\tilde{x}-x_b\|_2^2/(2\sigma^2)$ matching the Gaussian distance prior used for candidate generation. Because tree-based models are not differentiable, we used the gradient-free Nelder--Mead method. To reduce the risk of convergence to a local optimum, we performed multistart optimization from 20 initial points: $x_b$ and 19 points generated using the same proposal scale as in the importance-sampling procedure. We retained the solution with the smallest objective value, allowing at most 3,000 iterations per run. The proposition concerns the exact posterior mode, whereas the reported Gibbs MAP is selected from a finite Monte Carlo sample using an estimated density. Consequently, numerical agreement with DirectOpt is not guaranteed: differences may arise from finite candidate generation, resampling, density estimation, and local numerical optimization. DirectOpt is therefore reported as an objective-matched reference rather than as a validation test that must coincide with the approximate Gibbs MAP.

For the evaluation metrics, we used $n_{\mathrm{perturb}}=200$ perturbations for robustness $Rb$ and $q=20$ neighbors for plausibility $Plu$. The success thresholds $\varepsilon_{rb}$ and $\varepsilon_{sp}$ were set to 0.25 in 2D and 0.35 in 10D, and the tail probability was set to $\alpha=0.1$.

\subsubsection{Two-Dimensional Data (Figures 1, 2, and Table 1)}

\begin{figure}[H]
  \centering
  \includegraphics[width=0.7\linewidth]{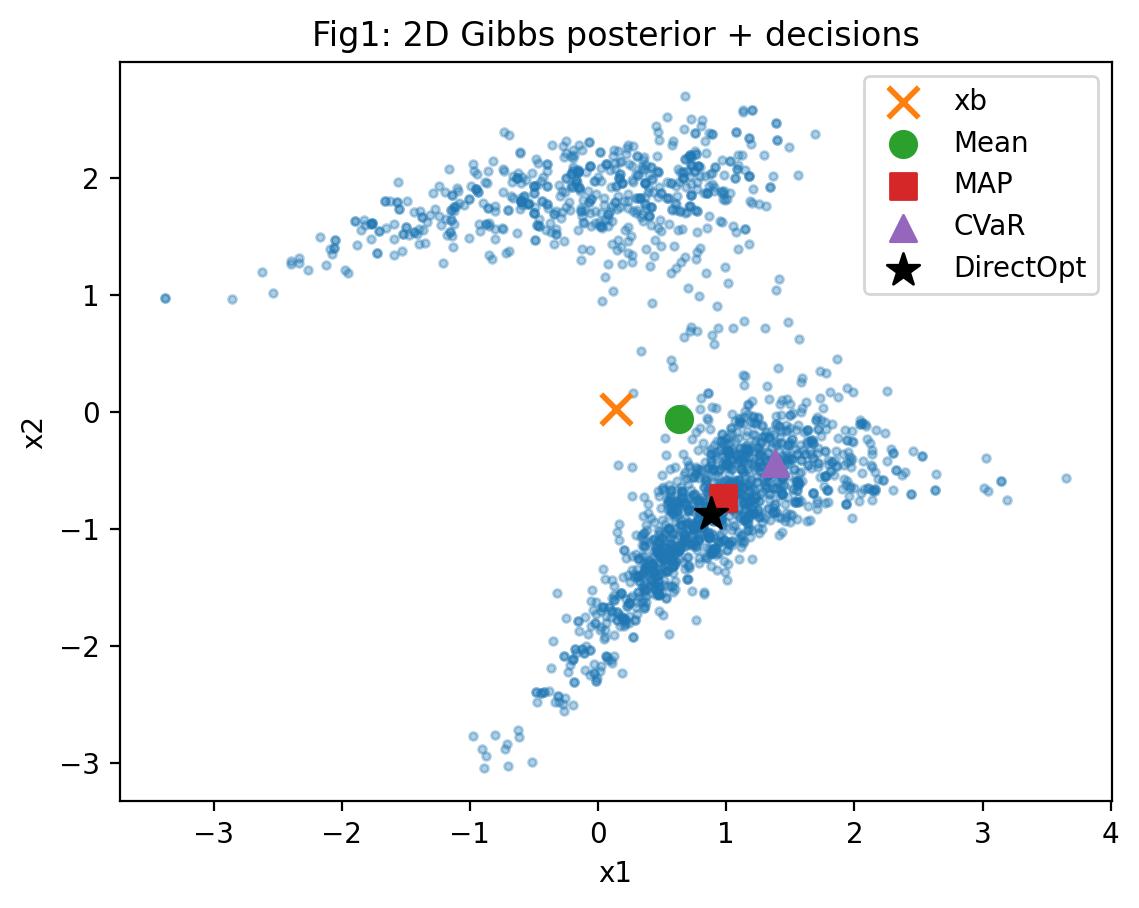}
  \caption{2D Gibbs posterior and CE decisions (MAP, Mean, CVaR-CE, DirectOpt$\star$)}
  \label{fig:fig1}
\end{figure}

Figure~\ref{fig:fig1} shows samples from the two-dimensional Gibbs posterior (blue points) and the CE selected by each decision rule. The posterior is multimodal and curved, reflecting the geometry of the regions that attain the target value $y^*$ (the 90th percentile of the training outcomes). Mean (green) lies between the base point and the main low-loss cluster and consequently has a large goal-attainment loss. MAP (red), CVaR-CE (purple), and DirectOpt (black star) lie within or near the main cluster, although they select different locations and therefore exhibit different loss--distance--robustness trade-offs (Table~\ref{tab:table1}).

\begin{figure}[H]
  \centering
  \includegraphics[width=0.7\linewidth]{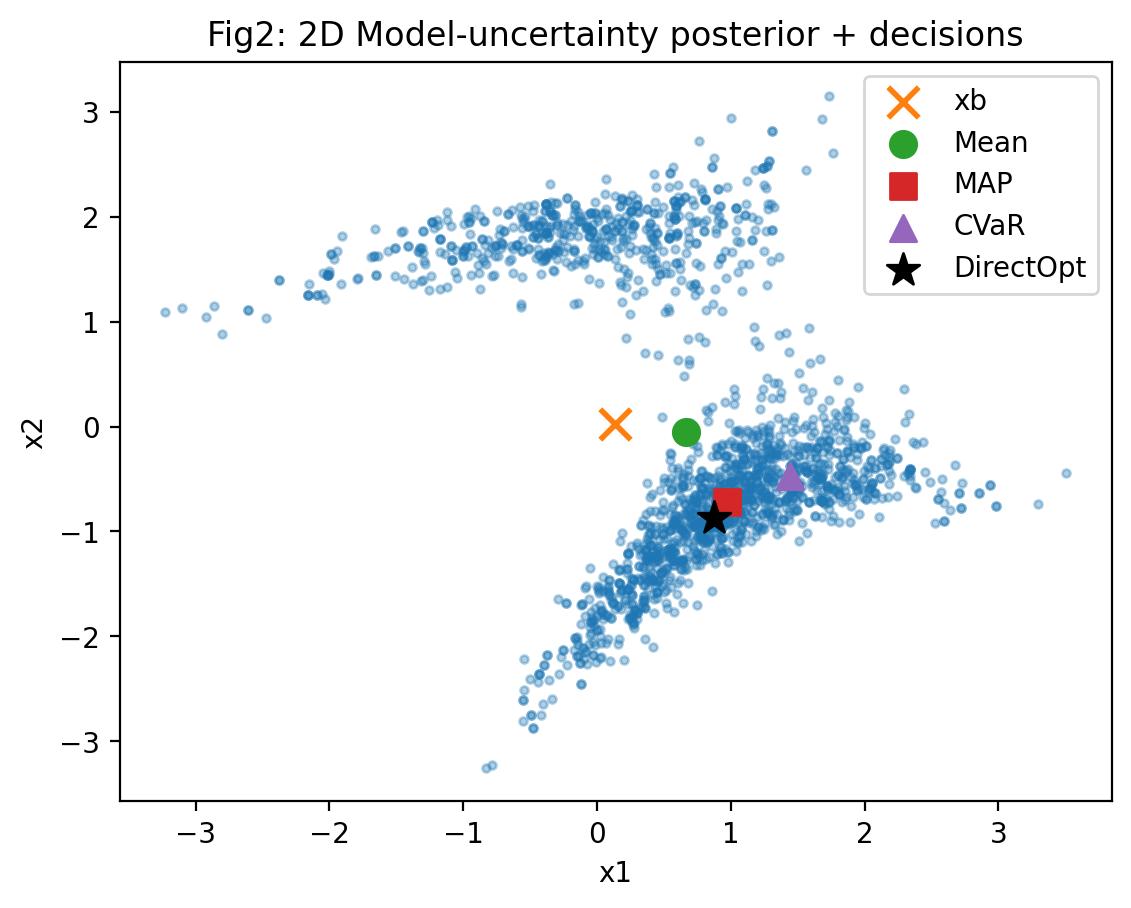}
  \caption{2D model-uncertainty CE posterior and CE decisions (MAP, Mean, CVaR-CE, DirectOpt$\star$)}
  \label{fig:fig2}
\end{figure}

Figure~\ref{fig:fig2} shows the posterior distribution of the model-uncertainty CE (a weighted mixture of the top three models) on the same data. Compared with the Gibbs posterior (Figure~\ref{fig:fig1}), mixing multiple models slightly increases the spread of the posterior distribution ($Stability$: 2.511 vs.\ 2.464) and slightly lowers the success probability ($SP$: 0.434 vs.\ 0.443). The decision points remain in broadly similar regions. For CVaR-CE, the model-uncertainty mixture yields higher robustness in this run ($Rb$: 0.555 vs.\ 0.465), illustrating that incorporating model uncertainty need not move every point-level metric in a uniformly adverse direction.

\begin{table}[H]
  \centering
  \caption{Two-dimensional simulated data: comparison of CE evaluation metrics ($VI_{x_j}$: distributional variable importance)}
  \label{tab:table1}
  \resizebox{\textwidth}{!}{
  \begin{tabular}{llrrrrrrrrr}
    \toprule
    Method & Decision rule & $L_{pt}$ & $D_{pt}$ & $Rb$ & $Plu$ & $SP$ & $Tail$ & $Stability$ & $VI_{x_1}$ & $VI_{x_2}$ \\
    \midrule
    Gibbs & Mean & 3.6534 & 0.5024 & 0.000 & 0.0941 & 0.443 & 1.7388 & 2.4636 & 0.808 & 1.167 \\
     & MAP & 0.4333 & 1.1313 & 0.455 & 0.1351 &  &  &  &  &  \\
     & CVaR & 0.1162 & 1.3255 & 0.465 & 0.1456 &  &  &  &  &  \\
    ModelUnc & Mean & 3.4477 & 0.5337 & 0.000 & 0.0978 & 0.434 & 1.9443 & 2.5109 & 0.849 & 1.145 \\
     & MAP & 0.4752 & 1.1205 & 0.420 & 0.1313 &  &  &  &  &  \\
     & CVaR & 0.0531 & 1.3980 & 0.555 & 0.1510 &  &  &  &  &  \\
    DirectOpt & --- & 0.0116 & 1.1616 & 0.485 & 0.1305 &  &  &  &  &  \\
    \bottomrule
  \end{tabular}
  }
  \par\vspace{2pt}{\footnotesize\textit{SP, Tail, Stability, and $VI_{x_j}$ are metrics of the posterior distribution as a whole and are common across decision rules within the same method; they are therefore shown only in the first row of each method group.}}
\end{table}

Table~\ref{tab:table1} reveals a clear trade-off among the decision rules. Mean has the shortest change distance but a large goal-attainment loss (Gibbs: $L_{pt}=3.653$; ModelUnc: 3.448) and zero perturbation robustness. MAP lowers the loss substantially while retaining a moderate distance. CVaR-CE further improves goal attainment and robustness at a modest additional distance. DirectOpt attains the smallest loss (0.012) with a distance similar to MAP and robustness of 0.485. The difference between DirectOpt and the approximate Gibbs MAP should not be read as contradicting the proposition: DirectOpt minimizes the continuous objective directly, whereas the reported MAP depends on finite posterior sampling and density estimation.

The $VI_{x_1}$ and $VI_{x_2}$ columns of Table~\ref{tab:table1} report the expected change $VarImp_j$ in each variable under each posterior. Under the Gibbs posterior, $VI_{x_1} \approx 0.808$ and $VI_{x_2} \approx 1.167$, indicating a larger change in $x_2$. In the data-generating process $y = 2.0\sin(x_1) + 0.8x_2^2 - 1.2x_1x_2$, $x_2$ affects $y$ through both a quadratic term and an interaction term, which may explain this larger posterior movement.

\subsubsection{Ten-Dimensional Data (Table 2)}

\begin{table}[H]
  \centering
  \caption{Ten-dimensional simulated data: comparison of CE evaluation metrics}
  \label{tab:table2}
  \resizebox{\textwidth}{!}{
  \begin{tabular}{llrrrrrrr}
    \toprule
    Method & Decision rule & $L_{pt}$ & $D_{pt}$ & $Rb$ & $Plu$ & $SP$ & $Tail$ & $Stability$ \\
    \midrule
    Gibbs & Mean & 1.0051 & 0.9314 & 0.190 & 2.2121 & 0.601 & 1.4130 & 9.6905 \\
     & MAP & 1.8156 & 0.9230 & 0.000 & 2.0186 &  &  &  \\
     & CVaR & 0.0439 & 3.2529 & 0.975 & 3.2542 &  &  &  \\
    ModelUnc & Mean & 0.9790 & 0.9015 & 0.295 & 2.2075 & 0.494 & 2.1290 & 9.6022 \\
     & MAP & 0.4047 & 1.7650 & 0.640 & 2.1765 &  &  &  \\
     & CVaR & 0.0005 & 2.8439 & 0.950 & 3.6927 &  &  &  \\
    DirectOpt & --- & 0.0170 & 1.4687 & 0.540 & 2.2095 &  &  &  \\
    \bottomrule
  \end{tabular}
  }
  \par\vspace{2pt}{\footnotesize\textit{SP, Tail, and Stability are metrics of the posterior distribution as a whole and are common across decision rules within the same method; they are therefore shown only in the first row of each method group.}}
\end{table}

Table~\ref{tab:table2} shows that the trade-off also depends on how the point decision is extracted in ten dimensions. Gibbs CVaR-CE attains low loss ($L_{pt}=0.044$) and high robustness ($Rb=0.975$), while ModelUnc CVaR-CE attains near-zero loss (0.0005) and robustness of 0.950, both at comparatively large distances. DirectOpt also attains low loss (0.017) at a shorter distance than either CVaR decision. In contrast, the density-based Gibbs MAP approximation has high loss (1.816) and zero robustness in this run, whereas the ModelUnc MAP performs better ($L_{pt}=0.405$, $Rb=0.640$). This instability is consistent with the difficulty of mode estimation from finite samples in high dimensions and underscores that the theoretical MAP equivalence does not imply equality between DirectOpt and a numerically estimated sample mode. The Gibbs success probability is higher in 10D than in 2D (0.601 vs.\ 0.443), but the candidate count and success threshold $\varepsilon_{sp}$ differ, so this difference should not be attributed solely to dimensionality.

\begin{table}[H]
  \centering
  \caption{Ten-dimensional simulated data: distributional variable importance $VI_{x_j} = \mathbb{E}[|\tilde{x}_j - x_{b,j}|]$}
  \label{tab:varimp_sim}
  \resizebox{\textwidth}{!}{
  \begin{tabular}{lrrrrrrrrrr}
    \toprule
    Method & $x_{1}$ & $x_{2}$ & $x_{3}$ & $x_{4}$ & $x_{5}$ & $x_{6}$ & $x_{7}$ & $x_{8}$ & $x_{9}$ & $x_{10}$ \\
    \midrule
    Gibbs & 0.843 & 1.044 & 0.763 & 0.836 & 0.807 & 0.779 & 0.818 & 0.770 & 0.789 & 0.805 \\
    ModelUnc & 0.849 & 0.984 & 0.789 & 0.807 & 0.802 & 0.793 & 0.796 & 0.797 & 0.792 & 0.799 \\
    \bottomrule
  \end{tabular}
  }
\end{table}

Table~\ref{tab:varimp_sim} shows the distributional variable importance $VarImp_j = \mathbb{E}[|\tilde{x}_{j} - x_{b,j}|]$ for the ten-dimensional data. Under the Gibbs posterior, the noise variables $x_8$--$x_{10}$ provide a prior-driven baseline of approximately 0.77--0.81. Most other variables are close to this range, while $x_2$ has the largest expected change (1.044). Because the data-generating process contains the quadratic term $0.8x_2^2$, this upward deviation indicates that movement in $x_2$ is especially prominent in constructing counterfactuals for this base point. The remaining coordinates should be interpreted cautiously: values near the baseline can reflect the spread of candidate generation as well as weak or diffuse loss constraints, so $VarImp_j$ is most informative when read relative to the noise-variable baseline rather than as an absolute importance score.

\subsubsection{Sensitivity Analysis for $\eta$ (Figure 3)}

\begin{figure}[H]
  \centering
  \includegraphics[width=0.48\linewidth]{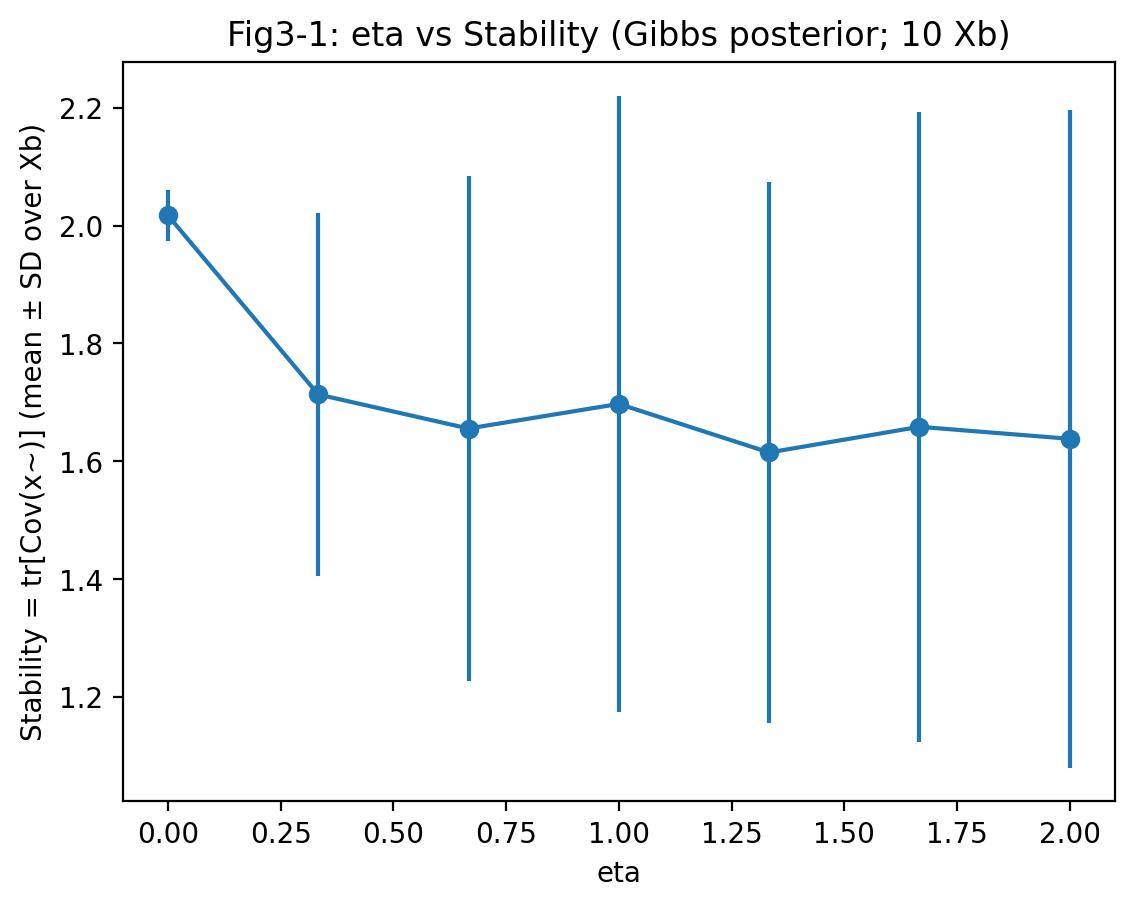}
  \hfill
  \includegraphics[width=0.48\linewidth]{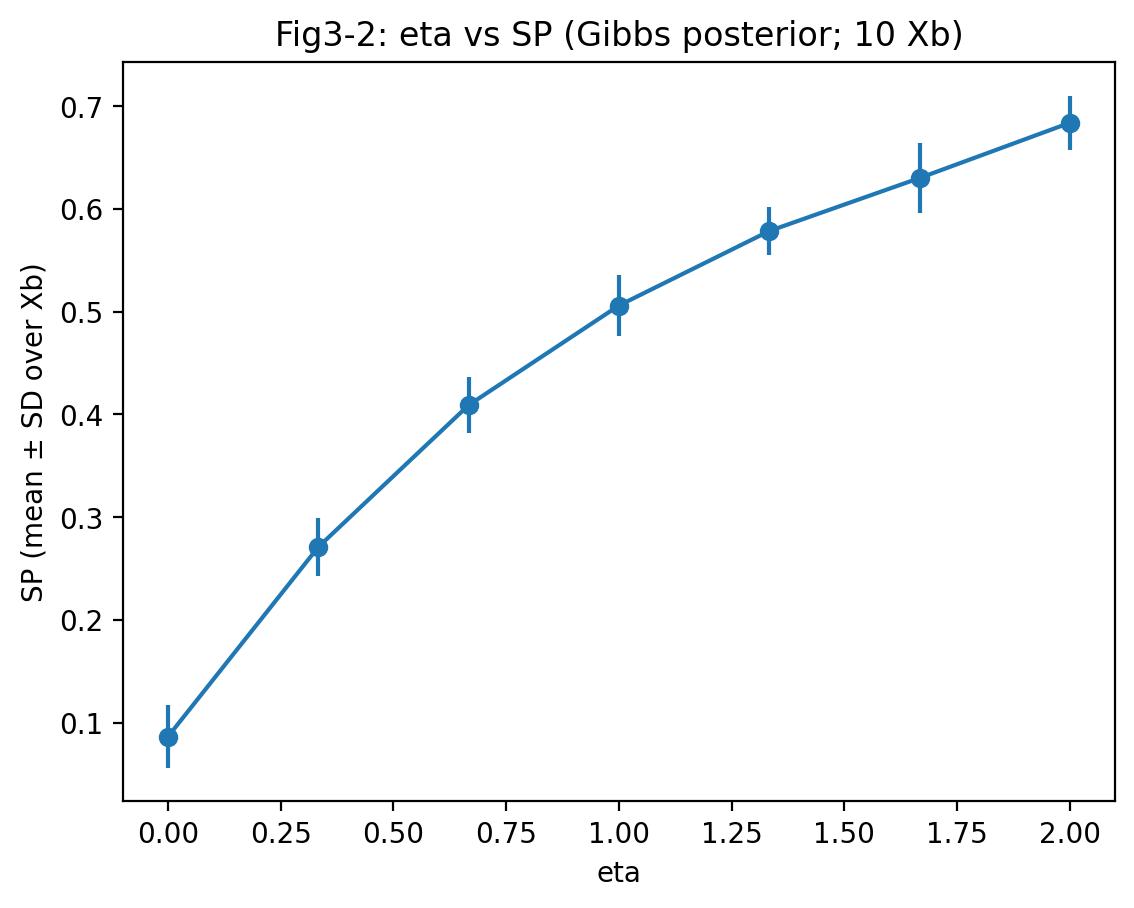}
  \caption{$\eta$-sensitivity analysis: mean $\pm$ SD of stability (left) and success probability (right) (Gibbs posterior; 5 randomly chosen base points $x_b$)}
  \label{fig:fig3}
\end{figure}

Figure~\ref{fig:fig3} shows the trend in the stability (left) and success probability (right) of the posterior distribution as $\eta$ varies. The mean success probability increases monotonically from about 0.09 at $\eta=0$ to 0.68 at $\eta=2$. Mean stability drops sharply at first (from about 2.02 to 1.71 by $\eta=1/3$) and then fluctuates around 1.6--1.7 rather than decreasing monotonically. Thus, larger $\eta$ more consistently improves concentration on the success region, while the diversity metric exhibits a plateau with Monte Carlo and base-point variation.

\begin{figure}[H]
  \centering
  \includegraphics[width=0.6\linewidth]{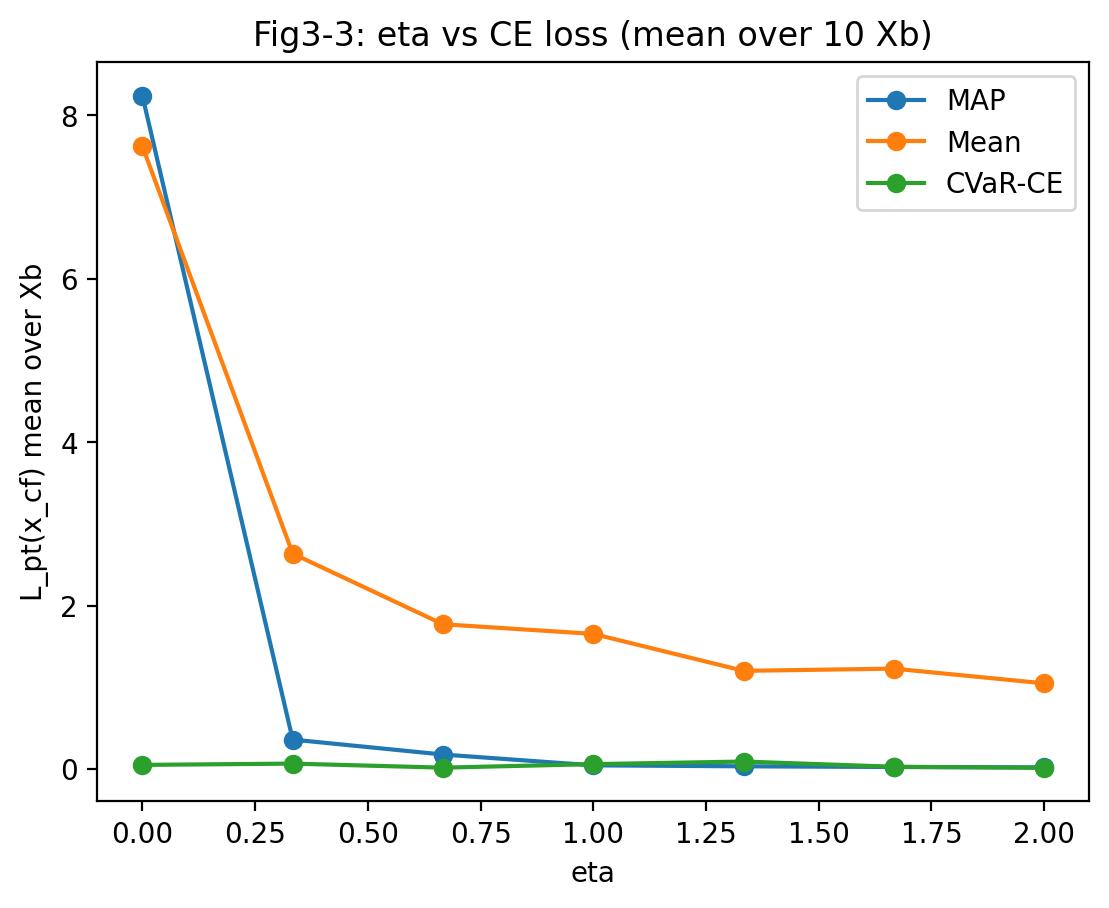}
  \caption{$\eta$-sensitivity analysis: mean CE loss (MAP, Mean, CVaR-CE)}
  \label{fig:fig3b}
\end{figure}

Figure~\ref{fig:fig3b} shows how the mean CE loss $L_{pt}$ for each decision rule changes with $\eta$. MAP loss falls sharply from 8.24 at $\eta=0$ to 0.045 at the default $\eta=1$ and remains near zero thereafter. CVaR-CE maintains a low but non-monotone loss throughout the grid. Mean loss decreases from 7.62 to 1.05 as $\eta$ rises to 2, but remains appreciably larger than the MAP and CVaR losses because the posterior mean can lie outside a curved or multimodal success region. These results support $\eta=1$ as a reasonable working value in this experiment while also showing that the decision rules respond differently to temperature.

\subsection{Real Data (Google Trends, One Piece--Related)}

Unlike the simulation, in which the true function is known, the real-data experiments focus on verifying whether the proposed method is practically usable in real-world settings where the true function is unknown. We therefore leave the visualization of the posterior distribution and the $\eta$-sensitivity analysis to the results on simulated data, and focus here on comparing the evaluation metrics across decision rules.

The real-data analysis uses monthly Japanese search-trend data obtained from Google Trends (January 2016 to December 2025; 119 observations). The target variable is search interest in ``One Piece'' (\textit{Wanpiisu}). We retrieved the manga/anime topic through Google Trends' topic-selection feature, thereby excluding searches for the unrelated clothing item. The explanatory variables are one-month-lagged (\textit{lag1}) search trends for ten related terms: Golden Week, summer vacation, winter vacation, New Year, \textit{Jump} (the magazine), LINE Manga, free manga, recommended manga, anime, and Netflix (Table~\ref{tab:desc_stats}). In the 10D analysis, these variables correspond, in order, to $x_1,\ldots,x_{10}$. The 2D analysis uses only $x_1'$, summer vacation\_lag1, and $x_2'$, \textit{Jump}\_lag1.

As in the simulation, we trained the prediction models using FLAML AutoML, evaluated them by MSE, and retained the top $K=3$ models. Table~\ref{tab:leaderboard_real} lists the selected models.

\begin{table}[H]
  \centering
  \caption{Real data: top three models selected by AutoML (ordered by MSE)}
  \label{tab:leaderboard_real}
  \begin{tabular}{llr}
    \toprule
    Data & Model & MSE \\
    \midrule
    2D & ExtraTrees & 58.08 \\
     & LightGBM & 60.34 \\
     & Random Forest & 62.80 \\
    \midrule
    10D & Random Forest & 51.93 \\
     & HistGradientBoosting & 53.35 \\
     & XGBoost & 54.50 \\
    \bottomrule
  \end{tabular}
  \begin{minipage}{0.8\linewidth}
    \footnotesize Note: MSE is the squared error of the Google Trends index (0--100 scale).
  \end{minipage}
\end{table}

The target $y^*$ was set to the 90th percentile of the observed outcomes, and squared-error loss was used. We generated $N_{\mathrm{cand}}=2\times10^4$ candidates in 2D and $N_{\mathrm{cand}}=3.5\times10^4$ in 10D, with $N_{\mathrm{samp}}=2\times10^3$ posterior samples in both cases. Although the 10D simulation used $N_{\mathrm{samp}}=3\times10^3$, the real-data analysis repeats the procedure with every eligible observation as $x_b$; we therefore used a common sample size to facilitate comparison between 2D and 10D. To reflect the 0--100 Google Trends scale, we set the proposal scale to $\sigma=10.0$ in 2D and $\sigma=8.0$ in 10D. The evaluation settings were adjusted accordingly: $\varepsilon_{sp}=\varepsilon_{rb}=15.0$ (equivalent under squared-error loss to an absolute prediction error of approximately 3.9 points), $\alpha=0.1$, and $\mathcal{D}_r=\mathcal{N}(0,5.0^2I)$ in both dimensions. For CVaR-CE, we used $\tau=0.9$, $n_{\mathrm{perturb}}=32$, and at most 400 candidate points. For robustness $Rb$, we used $n_{\mathrm{perturb}}=100$; for plausibility $Plu$, we used $q=10$ neighbors. Candidate points were not clipped to the feasible range $[0,100]$. The base points $x_b$ comprised all observations whose predicted values were below $y^*$ (2D: $n_{xb}=110$; 10D: $n_{xb}=106$), and we report the mean $\pm$ SD of each metric. MAP density estimation followed the simulation procedure: KDE in 2D and $k$-nearest-neighbor density estimation with $k=40$ in 10D.

\begin{table}[H]
  \centering
  \caption{Descriptive statistics of the real data (Google Trends)}
  \label{tab:desc_stats}
  \begin{tabular}{lrrrrr}
    \toprule
    Variable & Mean & SD & Min & Median & Max \\
    \midrule
    One Piece & 30.46 & 10.27 & 18.00 & 28.00 & 100.00 \\
    Golden Week\_lag1 & 6.94 & 14.28 & 0.00 & 1.00 & 100.00 \\
    Summer vacation\_lag1 & 24.50 & 29.85 & 4.00 & 9.00 & 100.00 \\
    Winter vacation\_lag1 & 14.39 & 20.13 & 2.00 & 6.00 & 100.00 \\
    New Year\_lag1 & 8.75 & 19.40 & 0.00 & 1.00 & 98.00 \\
    Jump\_lag1 & 67.29 & 14.98 & 38.00 & 66.00 & 100.00 \\
    LINE Manga\_lag1 & 66.41 & 15.30 & 36.00 & 66.00 & 100.00 \\
    Free manga\_lag1 & 66.05 & 16.97 & 40.00 & 66.00 & 100.00 \\
    Recommended manga\_lag1 & 59.03 & 16.81 & 30.00 & 57.00 & 100.00 \\
    Anime\_lag1 & 64.37 & 12.73 & 46.00 & 61.00 & 100.00 \\
    Netflix\_lag1 & 51.76 & 22.85 & 14.00 & 59.00 & 89.00 \\
    \bottomrule
  \end{tabular}
\end{table}

\begin{table}[H]
  \centering
  \caption{Real data (2D): mean $\pm$ SD across all samples ($VI_{x_j}$: distributional variable importance)}
  \label{tab:real2d}
  \resizebox{\textwidth}{!}{
  \begin{tabular}{llrrrrrrrrr}
    \toprule
    Method & Decision rule & $L_{pt}$ & $D_{pt}$ & $Rb$ & $Plu$ & $SP$ & $Tail$ & $Stability$ & $VI_{x_1'}$ & $VI_{x_2'}$ \\
    \midrule
    Gibbs & Mean & 29.395$\pm$115.667 & 19.134$\pm$6.809 & 0.510$\pm$0.254 & 10.080$\pm$4.490 & 1.000$\pm$0.000 & 2.054$\pm$0.637 & 129.296$\pm$91.664 & 11.222$\pm$4.630 & 16.498$\pm$5.649 \\
     & MAP & 1.040$\pm$0.684 & 19.320$\pm$8.345 & 0.494$\pm$0.200 & 9.286$\pm$5.346 &  &  &  &  &  \\
     & CVaR & 1.532$\pm$0.953 & 29.292$\pm$9.398 & 0.759$\pm$0.257 & 13.514$\pm$6.580 &  &  &  &  &  \\
    ModelUnc & Mean & 47.255$\pm$151.110 & 19.776$\pm$7.071 & 0.502$\pm$0.276 & 10.814$\pm$3.539 & 0.826$\pm$0.167 & 52.008$\pm$52.741 & 139.366$\pm$50.700 & 13.446$\pm$4.149 & 15.656$\pm$5.338 \\
     & MAP & 7.115$\pm$24.843 & 19.789$\pm$9.048 & 0.505$\pm$0.227 & 10.102$\pm$5.008 &  &  &  &  &  \\
     & CVaR & 10.786$\pm$28.001 & 29.401$\pm$9.099 & 0.731$\pm$0.277 & 14.239$\pm$7.146 &  &  &  &  &  \\
    \bottomrule
  \end{tabular}
  }
  \par\vspace{2pt}{\footnotesize\textit{SP, Tail, Stability, and $VI_{x_j}$ are metrics of the posterior distribution as a whole and are common across decision rules within the same method; they are therefore shown only in the first row of each method group.}}
\end{table}

Table~\ref{tab:real2d} reports the 2D Google Trends results as the mean $\pm$ SD across 110 base points. For the Gibbs posterior, $SP=1.000\pm0.000$, meaning that all posterior samples satisfied the success criterion for every base point. The change distance for Gibbs CVaR-CE, $D_{pt}=29.3\pm9.4$, is much larger than in the simulation ($\approx 3.2$), reflecting the 0--100 scale of the Google Trends variables. Gibbs MAP attains a low loss of $L_{pt}=1.040\pm0.684$ while limiting the amount of change. The Mean decision, by contrast, has $L_{pt}=29.4\pm115.7$, indicating unstable goal attainment. Relative to Gibbs, ModelUnc has a lower $SP$ (0.826 vs.\ 1.000) and a substantially higher $Tail$ (52.0 vs.\ 2.1), suggesting that mixing multiple models produces a heavier loss tail. The $VI_{x_j}$ columns report the expected changes in summer vacation\_lag1 ($x_1'$) and \textit{Jump}\_lag1 ($x_2'$). Under the Gibbs posterior, $VI_{x_1'} \approx 11.2$ and $VI_{x_2'} \approx 16.5$, averaged over $n_{xb}=110$, indicating that changes in \textit{Jump}-related search demand contribute more strongly to the CE for the One Piece search trend.

\begin{table}[H]
  \centering
  \caption{Real data (10D): mean $\pm$ SD across all samples}
  \label{tab:real10d}
  \resizebox{\textwidth}{!}{
  \begin{tabular}{llrrrrrrr}
    \toprule
    Method & Decision rule & $L_{pt}$ & $D_{pt}$ & $Rb$ & $Plu$ & $SP$ & $Tail$ & $Stability$ \\
    \midrule
    Gibbs & Mean & 24.881$\pm$34.005 & 13.547$\pm$7.643 & 0.451$\pm$0.194 & 33.400$\pm$9.638 & 1.000$\pm$0.000 & 1.505$\pm$0.400 & 607.938$\pm$72.132 \\
     & MAP & 0.345$\pm$0.544 & 20.456$\pm$10.013 & 0.464$\pm$0.196 & 36.569$\pm$9.783 &  &  &  \\
     & CVaR & 0.914$\pm$1.638 & 33.701$\pm$6.918 & 0.761$\pm$0.196 & 44.213$\pm$10.041 &  &  &  \\
    ModelUnc & Mean & 20.446$\pm$24.708 & 13.826$\pm$7.426 & 0.449$\pm$0.194 & 33.362$\pm$9.635 & 0.841$\pm$0.118 & 21.433$\pm$11.446 & 617.107$\pm$85.673 \\
     & MAP & 5.518$\pm$13.223 & 20.449$\pm$9.899 & 0.507$\pm$0.178 & 36.558$\pm$9.203 &  &  &  \\
     & CVaR & 3.909$\pm$9.061 & 34.689$\pm$6.427 & 0.767$\pm$0.197 & 44.868$\pm$9.266 &  &  &  \\
    \bottomrule
  \end{tabular}
  }
  \par\vspace{2pt}{\footnotesize\textit{SP, Tail, and Stability are metrics of the posterior distribution as a whole and are common across decision rules within the same method; they are therefore shown only in the first row of each method group.}}
\end{table}

Table~\ref{tab:real10d} reports the 10D results. All Gibbs posterior samples satisfied the success criterion for every base point ($SP=1.000\pm0.000$), and Gibbs MAP achieved a low loss of $L_{pt}=0.345\pm0.544$. The Mean decision had the smallest distance ($D_{pt}=13.5$) but a comparatively large and variable loss ($L_{pt}=24.9\pm34.0$), showing that the limitations of the posterior centroid in high dimensions also arise in the real-data analysis. Relative to Gibbs, ModelUnc had a lower $SP$ (0.841 vs.\ 1.000) and a substantially higher $Tail$ (21.4 vs.\ 1.5). As in the 2D analysis, mixing multiple models therefore produced a heavier loss tail.

\begin{table}[H]
  \centering
  \caption{Real data (10D): distributional variable importance $VI_{x_j}$}
  \label{tab:varimp_real10d}
  \resizebox{\textwidth}{!}{
  \begin{tabular}{lrrrrrrrrrr}
    \toprule
    Method & $x_{1}$ & $x_{2}$ & $x_{3}$ & $x_{4}$ & $x_{5}$ & $x_{6}$ & $x_{7}$ & $x_{8}$ & $x_{9}$ & $x_{10}$ \\
    \midrule
    Gibbs & 6.475 & 7.448 & 6.560 & 6.441 & 6.652 & 6.922 & 11.673 & 7.899 & 6.292 & 6.581 \\
    ModelUnc & 6.443 & 7.162 & 6.482 & 6.313 & 6.814 & 7.109 & 11.723 & 8.126 & 6.562 & 6.597 \\
    \bottomrule
  \end{tabular}
  }
\end{table}

For the ten-dimensional data, Table~\ref{tab:varimp_real10d} shows $VarImp_j$ for each variable (the maximum is $x_7$ = free manga\_lag1, at approximately 11.7, followed by $x_8$ = recommended manga\_lag1, at approximately 7.9). A larger value means that the variable plays a more important role in constructing the counterfactual.

These real-data experiments have several practical implications. First, the Mean decision exhibits substantial variation in $L_{pt}$ (e.g., $SD=115.7$ in 2D), making it an unstable point estimate that should be used with caution. As noted in Section 2.5.1(b), this behavior reflects a known limitation of point summaries: under a multimodal or curved posterior, the mean can lie in a low-density region. It does not indicate a defect in the posterior distribution itself. MAP, by contrast, attains a low loss while limiting the amount of change, whereas CVaR-CE improves robustness at the cost of a larger change. The choice between MAP and CVaR-CE should therefore depend on the intervention cost and the consequences of failing to attain the target. Second, ModelUnc has a substantially higher $Tail$ than the single-model Gibbs posterior, suggesting that accounting for model-selection uncertainty can reveal greater tail risk than an analysis based on a single model. Third, the distributional variable importance $VarImp_j$ identifies \textit{Jump}\_lag1, free manga\_lag1, and recommended manga\_lag1 as strong contributors to the counterfactuals. Monitoring these search-demand indicators -- interest in the serialization magazine and free or recommendation-oriented manga services -- may therefore help inform the timing of measures intended to increase interest in ``One Piece.''

\section{Related Work}

Several studies have examined CEs within a Bayesian framework. As noted in Section 1, \textcite{raman2023bayesian} formulate counterfactual generation as a probabilistic model of perturbations. Their hierarchical Bayesian framework incorporates validity and proximity as likelihood terms and samples diverse counterfactuals from the posterior. In particular, their population--subgroup--instance hierarchy improves robustness by steering counterfactuals toward high-density regions of the data and supports fairness assessments based on comparisons of recourse costs across protected subgroups. Although internally coherent as a probabilistic model, the method requires a computationally demanding likelihood-based inference procedure. Its use of HMC/NUTS also requires a differentiable classifier. By contrast, the DP-GBCE proposed here is based on generalized Bayes and constructs a pseudo-posterior that combines goal-attainment loss with change cost, without requiring an explicit generative model or a specified likelihood. It can therefore represent counterfactuals as a distribution for differentiable models such as neural networks, nondifferentiable learners such as random forests, and black-box predictors accessed through external APIs, provided that the loss can be evaluated. It also enables multiple decision rules to be compared within a common framework and is thus model agnostic.

Another relevant study is \textcite{nguyen2022robust}, who define recourse for black-box classifiers in Bayesian terms and derive a framework that identifies points with a high probability of receiving the desired prediction by minimizing posterior odds. To address the degradation of recourse validity under future changes to the classifier (model shift), they introduce a min--max optimization problem over ambiguity sets of class-conditional distributions defined by the Wasserstein distance, thereby obtaining recourse with a high success probability even in the worst case. Their approach, however, does not construct a distribution over counterfactuals from a prior and likelihood. Instead, it uses local sampling and density estimation to evaluate the odds of success and selects a single recourse point with a high success probability.

The DP-GBCE proposed here instead places a probability distribution directly over counterfactual candidates. Through decision rules such as MAP, the posterior mean, and CVaR, it provides a unified way to assess and summarize uncertainty arising from multimodality or dispersion in the success region. Thus, whereas \textcite{nguyen2022robust} focus on robustness to model changes, our focus is robust decision-making under uncertainty about the success region.

Research on distributions of CEs can be divided broadly into two categories: work that considers a distribution of counterfactuals for a single observation and work that evaluates distributions of counterfactuals across a population. The first category seeks to represent counterfactual diversity and uncertainty explicitly, thereby avoiding the instability associated with reliance on a single optimum. It includes methods that generate diverse sets of CEs and methods based on probabilistic generative models \parencite[e.g.,][]{mothilal2020explaining, raman2023bayesian}. The second category focuses on the fairness of recourse at the population level, including comparisons of the distribution of recourse costs (corresponding to $D_{pt}$ in this study) across subgroups and analyses of institutional fairness \parencite[e.g.,][]{ustun2019actionable}. Our study belongs to the first category: it formulates the counterfactual for a single observation as a probability distribution and incorporates decision rules directly into this distributional CE framework.

The robustness of a CE may be challenged by several factors. \textcite{jiang2024robust} classify robust CEs into four categories: (1) model changes (MC), or robustness to changes caused by retraining or distribution shift; (2) model multiplicity (MM), or uncertainty arising from the coexistence of multiple models with comparable performance; (3) noisy execution (NE), or noise and errors introduced when a CE is implemented; and (4) input changes (IC), including the consistency of explanations for similar inputs. This study primarily addresses robustness under MM. For this setting, \textcite{pawelczyk2020counterfactual} derive theoretical upper bounds on CE costs under predictive multiplicity based on disagreement among classifiers, and empirically evaluate the robustness of existing CE methods, including methods constrained to data-supported regions. The multi-objective optimization approach of \textcite{kinjo2025robust}, meanwhile, constructs a Pareto set by minimizing the losses for all models simultaneously and thereby makes the trade-off structure among solutions explicit. Our approach differs in that it integrates model uncertainty probabilistically through generalized Bayes and yields a distributional CE with embedded decision rules.

The relationship between our framework and both the multi-objective approach of \textcite{kinjo2025robust} and min--max approaches that provide worst-case guarantees requires clarification. The effective loss of the mixture posterior in Eq.~\eqref{eq:model-mixture},
$-\eta^{-1}\log\sum_k w_k e^{-\eta\ell_k}$, converges to the weighted expected loss $\sum_k w_k\ell_k$ as $\eta\to0$, corresponding to a linear scalarization of a multi-objective problem. As $\eta\to\infty$, however, it converges to $\min_k\ell_k$. The mixture posterior is therefore an OR-type aggregation that assigns mass to regions in which at least one model succeeds. Its direction of aggregation differs from the AND-type worst-case guarantee generally provided by min--max optimization, which seeks validity under every model. Consequently, our framework does not provide a worst-case guarantee. Instead, Bayesian marginalization exposes disagreement among models through greater posterior dispersion, reflected in increases in $Tail$ and $Stability$, as observed in Section 3. When worst-case robustness is required, a possible extension would replace the mixture with a loss-side aggregation such as $\exp(-\eta\max_k\ell_k)$. Establishing the properties of such an extension is left for future work.

\section{Discussion}

We first showed that a counterfactual explanation obtained by distance minimization is equivalent to the MAP estimate of a Gibbs posterior over CEs. This result provides a theoretical justification, from the perspective of Bayesian inference, for the conventional cost-minimization formulation. We then extended the framework by proposing a CE posterior that integrates model uncertainty probabilistically. For the resulting posterior, we introduced two decision rules beyond MAP: a Bayes decision that minimizes expected decision loss and CVaR-CE. This provides a unified framework for multiple decision-making objectives. Finally, we proposed metrics for evaluating both individual CEs and the posterior distribution as a whole, and quantified the trade-offs among the decision rules through experiments on simulated data and real Google Trends data.

The significance of these results can be summarized in four points. First, we theoretically established that CE distance minimization is equivalent to MAP estimation under a generalized Bayes Gibbs posterior. Conventional CEs have been justified on practical and philosophical grounds as minimum-cost changes that attain a target; by identifying this formulation as a special case of generalized Bayes, namely MAP estimation under a Gibbs measure, our study provides a broader theoretical foundation. Second, the proposed framework is model agnostic: it requires neither an explicit generative model nor a specified likelihood and can be applied to any learner, differentiable or otherwise, provided that its loss can be evaluated. It can therefore construct counterfactuals for black-box predictors that cannot be handled by methods requiring gradient information. Third, the posterior perspective places MAP, Bayes, and CVaR-CE decisions within a single framework, allowing the decision rule to be selected according to the structure of the problem, such as the shape of the success region or the cost of failure. Fourth, for model multiplicity, where several models have comparable predictive performance, the framework naturally yields a distributional CE that incorporates model uncertainty by mixing the model-specific posteriors using Bayesian weights.

The interpretation of the posterior $p(\tilde{x}|x_b,y^*)$ underlying these contributions warrants further discussion. The prior $p(\tilde{x}|x_b)$ is a candidate distribution based solely on proximity to $x_b$, whereas the prediction model $f$ supplies additional information by imposing the requirement that the target $y^*$ be attained. The posterior is constructed under this requirement. With a single prediction model (Section 2.3), the spread of the posterior reflects only the multiplicity of the success region, which would remain even if the true $f$ were known perfectly. With multiple prediction models (Section 2.4), an additional source of uncertainty arises because the models impose different requirements; in principle, this uncertainty can be reduced by acquiring more information. As discussed in Section 2.4, ModelUnc is a practical approximation to FG-GBCE and has the classical evidence-based structure in which uncertainty about $\theta$ can diminish as the amount of data $D$ increases. In the experiments in Section 3, ModelUnc produced larger values of $Tail$ and $Stability$ than the single-model Gibbs posterior. This reflects the wider set of candidates that remains when the models impose conflicting constraints, indicating that disagreement among the models is appropriately represented in the posterior. In this sense, the posterior is best understood not as an inference about an unknown fact, but as a framework for choosing among multiple valid solutions under a given constraint (the prediction model) using the tools of Bayesian decision theory, including MAP, expected-loss minimization, and CVaR.

This study also has several limitations. First, the efficiency of the importance-sampling approximation decreases as the dimension of $x$ increases. Second, we fixed the temperature parameter at $\eta=1$ and did not establish a theoretically grounded method for selecting it. Third, our analysis focused primarily on continuous features and did not accommodate binary features or constraints distinguishing mutable from immutable features. Finally, although we fixed the random seed in both the simulated- and real-data experiments to ensure reproducibility, FLAML's AutoML search depends on its time budget. The selected models and resulting estimates may therefore vary slightly with the budget.

Several directions for future work follow from these limitations. More scalable sampling methods, such as Metropolis--Hastings (MH), could improve approximation efficiency in high dimensions. The temperature parameter $\eta$ might be selected using theoretically grounded methods based on SafeBayes or PAC-Bayes. Extending the framework to binary features and to constraints on mutable and immutable features is also important. Beyond model multiplicity, the framework may be applicable to other forms of robustness, including noisy execution and input changes. A further direction is to use the posterior distribution over CEs directly in downstream decision-making, such as policy planning or personalized recommendations.

Finally, an important extension is to make the FG-GBCE described in Section 2.4, and its special case of a fully Bayesian CE, computationally practical. These formulations can represent epistemic uncertainty about $\tilde{x}$ in a more classical sense, but they require a nested procedure that combines sampling from the posterior over the model parameters $\theta$ with sampling over $\tilde{x}$. This procedure is computationally expensive, and tree-based models present the additional structural difficulty that they lack a fixed-dimensional continuous parameter vector $\theta$. Future work should develop approximations that overcome these practical barriers and investigate extensions such as $p(\tilde{x}|D)$ that incorporate $D$ to assess validity in terms of proximity to the observed data.

\printbibliography[title={References}]

\appendix
\section*{Appendix}

\subsection*{Variational Derivation of the Gibbs Posterior}

Rather than estimating the counterfactual $\tilde{x}$ as a single point, we treat it as a distribution $q(\tilde{x})$. Let $L(\tilde{x})=\ell(f(\tilde{x}),y^*)$ denote the loss measuring attainment of the target $y^*$, and introduce a prior $\pi(\tilde{x}\,|\,x_b)$ representing natural changes from the base point $x_b$. Following a variational principle, we seek to reduce the expected loss $\mathbb{E}_q[L(\tilde{x})]$ while controlling divergence from the prior $\pi(\tilde{x}\,|\,x_b)$ through $\mathrm{KL}(q\|\pi)$. The Donsker--Varadhan variational formula implies that, for any distribution $q$ absolutely continuous with respect to $\pi$ and any $\eta>0$,

\begin{equation}
\label{eq:dv-bound}
\mathbb{E}_{\tilde{x}\sim q}[L(\tilde{x})] + \frac{1}{\eta}\mathrm{KL}(q\|\pi) \,\geq\, -\frac{1}{\eta}\log \mathbb{E}_{\tilde{x}\sim\pi}\!\left[e^{-\eta L(\tilde{x})}\right].
\end{equation}

The expectation on the right-hand side is the normalizing constant (partition function), namely the prior expectation of the exponentially weighted loss. The functional on the left-hand side, which expresses the trade-off between loss and proximity to the prior, is uniformly bounded below by the right-hand side. Equality holds if and only if $q$ is the Gibbs posterior $q_\eta\propto\pi e^{-\eta L}$. Thus, the Gibbs posterior arises as the distribution that optimally balances the loss against divergence from the prior. Specifically,

\begin{equation}
\label{eq:gibbs-argmin}
q_\eta \in \arg\min_{q}\left\{\mathbb{E}_q[L(\tilde{x})] + \frac{1}{\eta}\mathrm{KL}(q\|\pi)\right\},
\end{equation}

with the closed-form expression

\begin{equation}
\label{eq:gibbs-closedform}
q_\eta(\tilde{x}) \propto \pi(\tilde{x}\,|\,x_b)\exp\{-\eta L(\tilde{x})\}.
\end{equation}

The parameter $\eta$ can therefore be interpreted as a temperature (or inverse-temperature) parameter controlling the trade-off between loss minimization and proximity to the prior, which represents the naturalness of a change. As $\eta$ increases, the distribution concentrates in low-loss regions; as it decreases, the distribution remains more diffuse and closer to the prior.

\end{document}